\documentclass[a4paper]{article}

\usepackage[utf8]{inputenc}
\usepackage[T1]{fontenc}

\usepackage{microtype}
\usepackage{graphicx}
\usepackage{subfigure}
\usepackage{booktabs} 

\usepackage[bookmarks=false]{hyperref}

\usepackage{amsmath}
\usepackage{amssymb}
\usepackage{mathtools}
\usepackage{amsthm}

\theoremstyle{plain}

\newtheorem{example}{Example}

\newtheorem{proposition}{Proposition}

\theoremstyle{definition}
\newtheorem{definition}{Definition}

\newcommand{\cinput}{\ensuremath{\mathbf{x}}}
\newcommand{\vech}{\ensuremath{\mathbf{h}}}
\newcommand{\vecx}{\ensuremath{\mathbf{x}}}
\newcommand{\vecy}{\ensuremath{\mathbf{y}}}
\newcommand{\coutput}{\ensuremath{y}}
\newcommand{\classifier}{\ensuremath{\operatorname{C}}}
\newcommand{\dom}{\ensuremath{\mathcal{D}}}
\newcommand{\attributionM}{\ensuremath{s_{\classifier,l}}}
\newcommand{\attributionMgrad}{\ensuremath{g_{\classifier,l}}}

\newcommand{\FT}[1]{\textcolor{black}{#1}}

\usepackage{todonotes}

\begin{document}

\title{Towards a Theory of Faithfulness: Faithful Explanations
           of Differentiable Classifiers over Continuous Data}


\author{
   Nico Potyka\\
   Imperial College London \\
  \texttt{n.potyka@imperial.ac.uk}
  \and
   Xiang Yin\\
   Imperial College London \\
  \texttt{x.yin20@imperial.ac.uk}
  \and
   Francesca Toni\\
   Imperial College London \\
  \texttt{f.toni@imperial.ac.uk}
}
\date{}

\maketitle

\begin{abstract}
There is broad agreement in the literature that explanation methods 
should be \emph{faithful} to the model that they explain, but 
faithfulness remains a rather vague term. We revisit faithfulness 
in the context of continuous data and 
propose two formal definitions of faithfulness for feature attribution 
methods. \emph{Qualitative faithfulness} demands that scores reflect the 
true qualitative effect (positive vs. negative) of the feature on the model 
and \emph{quanitative faithfulness} that the magnitude of scores reflect
the true quantitative effect. We discuss under which conditions these
requirements can be satisfied to which extent (local vs global).
As an application of the conceptual idea, we look at differentiable
classifiers over continuous data and characterize Gradient-scores
as follows: every qualitatively faithful feature attribution method is
qualitatively equivalent to Gradient-scores. Furthermore, 
if an attribution method is quantitatively faithful in the sense that 
changes of the output of the classifier are proportional to the scores of
features, then it is either equivalent to gradient-scoring or it is based 
on an inferior approximation of the classifier.
To illustrate the practical relevance of the theory, we experimentally
demonstrate that popular attribution methods can fail to give faithful
explanations in the setting where the data is continuous and the 
classifier differentiable. 
\end{abstract}

\section{Introduction}

Automatic decision making is increasingly driven by black-box
machine learning models.
However, their opaqueness raises questions about
fairness, reliability and safety.
Explanation methods aim at making
the decision process transparent \cite{adadi2018peeking}.
There is some debate about what an explanation should or should
not look like \cite{lipton2018mythos}.
On the one hand, it should be easily comprehensible
to a lay
\FT{person}. On the other, it should be \emph{faithful} to the model
\FT{, by explaining} the true reasoning of the system.

Even though faithfulness is an important property, it seems
that there is no universally agreed definition of what it actually is.
Since explanations come in very different forms, it seems that 
a formal definition has to depend on the nature of the explanation method.
Counterfactual explanations, for example, explain how an input
would have to be changed to change the 
\FT{model's output} \cite{wachter2017counterfactual}.
Rule-based approaches like Anchors \FT{\cite{ribeiro2018anchors}} try to explain the behaviour 
of a system by human-readable rules
.
Our focus here will be on attribution methods that
assign a numerical score to every feature
\FT{, reflecting} its positive
or negative influence on the 
\FT{model's output}.
Popular examples include
LIME, SHAP, SILO and MAPLE \cite{ribeiro2016should,lundberg2017unified,bloniarz2016supervised,plumb2018model}.

Since many attribution methods
compute scores by approximating the prediction model locally by an 
interpretable model,
faithfulness is often considered as the accuracy of the
approximation in a neighborhood of the input \cite{ribeiro2016should,plumb2018model}.
While this is a natural and pragmatic idea, it is not without flaws.
For example, consider a dataset for credit scoring and assume that 
the score is highly correlated with the income. Furthermore, assume that
the majority of samples from a particular ethnic group have a low
income. Then a prediction model may have learnt that all people
from this group should have a low score. However, the approximation
model may actually approximate the prediction model well when
using only the income and ignoring the ethnicity. The explanation
will then suggest that the prediction model is unbiased, while it
actually is, and the explanation is clearly not faithful.
Let us note that this problem is only amplified when trying to
sample from the data distribution instead of considering samples
outside of the distribution (that may or may not be unrealistic in practice).

Therefore, it seems important to have formal requirements
of faithfulness that can be tested or even
formally proved for explanation methods. For attribution methods,
the following two desiderata seem natural:
\begin{description}
\item[Qualitative Faithfulness:] A feature with positive (negative) score,
should positively (negatively) influence the output of the prediction model.
This idea has been called \emph{Demand Monotonicity} in \cite{SundararajanN20}
\item[Quantitative Faithfulness:] The magnitude of a feature's score,
should reflect its impact on the output.
\end{description}
These intuitive ideas can be formalized in different ways in different settings.
We will focus on continuous data here and apply tools from Real Analysis
to make the intuition precise.
Let us already note that whether or not these definitions can be satisfied
depend on the nature of the prediction model. For example, a linear prediction
model behaves uniformly on its domain, so that one can assign a globally
qualitatively and quantitatively faithful score to a feature. However, non-linear prediction models can change their qualitative as well as
quantitative behaviour in different regions, so that we can expect 
faithfulness only locally. 

As an application of the theory, we study properties of gradient-based
feature attribution that has been proposed in \cite{baehrens2010explain}.
The idea is to define the score of a feature as the partial derivative
of the prediction model with respect to the feature.
We will refer to this attribution method as \emph{Gradient-scoring} in
the following.
Gradient-scoring lost popularity in recent years due to its noisy behaviour in 
natural language and computer vision tasks \cite{heo2019fooling,ghorbani2019interpretation,wang2020gradient}.
However, we argue that the problems observed in these domains are not
due to inherent limitations of the gradient, but due to the discrete 
nature of the data (e.g. word occurence, integer pixel values). 

For qualitative faithfulness, we show that Gradient-scoring
gives the strongest guarantees that we can hope for (Propositions \ref{prop_grad_scor_faithful} and \ref{prop_mon_grad_global_faithful}) 
and that every other method with similar guarantees must be qualitatively equivalent to Gradient-scoring (Proposition \ref{prop_qual_faithfulness_and_grad}).
Our conceptualization of quantitative faithfulness is somewhat ambiguous
in the sense that the \emph{impact of a feature on the output} can be
quantified in different ways. We choose to demand that feature scores 
should be proportional to the change of the output when the input feature
is changed. We show that every attribution method that satisfies this property
is either equivalent to Gradient-scoring or is based on an inferior 
approximation of the prediction model (Proposition \ref{thm_quant_faithfulness_characterization} and \ref{prop_quant_faithfulness_convergence_rate}). 
Our interpretation of these results is that Gradient-scoring should be 
the method of choice in settings where the data is continuous and 
the prediction model differentiable.

To support this claim empirically, we investigate in Section \ref{sec_experiments} to which extent different methods can capture the 
true behaviour of logistic regression (known ground-truth) and 
compare gradient scores to other scores on classification problems with continuous features when using multilayer perceptrons (unknown ground-truth).
Our logistic regression experiments show, in particular, that LIME, 
SHAP and SILO can fail to give faithful explanations. However, overall,
LIME seems to be very robust. Given LIME's close conceptual relationship 
to the gradient
(it locally approximates the prediction model by an interpretable model),
it seems to be a natural extrapolation of Gradient-scoring to domains with
discrete features or non-differentiable prediction models.

\section{Related Work}

In the literature about attribution methods, several properties have been 
proposed to compare different approaches. 
\cite{lundberg2017unified} defined three properties
for local explanations (explanations of an output with respect to a given input) that characterize Shapley Additive Explanations. Intuitively,
they demand the following. 
\begin{enumerate}
\item \emph{Local Accuracy:} the output of a model
at an input, can be described by an affine function of the
input that is determined by the feature scores. 
\item \emph{Missingness:} the scores of features that are not present 
at an input are $0$.
\item \emph{Consistency:} If the impact of a feature on prediction model 
$M_1$
is always larger than the impact on model $M_2$, then the feature
score for $M_1$ should be larger than the one for $M_2$.
\end{enumerate}
While the properties and the characterization are interesting,
computing Shapley Additive Explanations exactly is often impractical
(because there are too many combinations) or even impossible (because there 
is an infinite number of combinations as soon as one feature domain is infinite). While various approximations have been proposed, it is not clear
to which extent the desirable properties can actually be maintained in practice. Our logistic regression experiments in Section \ref{sec_experiments}
indicate that Shapley scores can be highly misleading when being applied
to continuous domains.

\cite{sundararajan2017axiomatic} proposed some properties that 
characterize \emph{Integrated Gradients}, a variant of Gradient-scoring
that averages the gradient over a region instead of taking the gradient
at a single point. Most relevant for our work, the \emph{Sensitivity} property demands that a feature that can change the output,
will receive a non-zero score. Gradient-scores can violate
this property for non-linear prediction models because the gradient can
be $0$ at points where the qualitative impact of a feature changes from
positive to negative or vice versa. There are arguments for and against
this behaviour. One may argue that $0$ is misleading because the feature
can have an impact in a region close to the input. 
However, let us note that a smooth prediction model over continuous
features will also have small derivatives at points close to the input.
Therefore, the property seems more important for discrete domains or
non-smooth prediction models. One may also argue that, if the prediction
model changes its behaviour, averaging the impact can also be misleading
because it suggests monotonicity, where no monotonicity exists. Of course,
this can be said about attribution methods in general, as they assign
scores to features independently without considering their joint effects.
Let us note that the \emph{Dummy} property \cite{SundararajanN20} complements
\emph{Sensitivity} by demanding that a feature that cannot change the output,
will receive a zero score. This property is also satisfied by Gradient-scores
and is related to our definition of \emph{Strong Qualitative Faithfulness} that we will discuss in Section \ref{sec_qualitative_faithfulness}.

In the context of attribution methods, 
\emph{faithfulness} is often understood
as the accuracy of a local approximation when predicting the 
output of a prediction model
\cite{sanchez2015towards,ribeiro2016should,plumb2018model}.
As we already discussed in the introduction, this can be misleading
because a high accuracy does not guarantee that the explainer picked
up the actual behaviour of the prediction model.
Another problem with this definition is that it 
depends on the definition of a neighbourhood and a sampling strategy 
from which the test set is generated. 

Recent attribution methods like LIME and MAPLE  \cite{ribeiro2016should,plumb2018model}
are partially based on the idea of approximating a black-box model
locally by a linear model in order to use the coefficients of the
linear models as feature scores. Let us note that in a setting with
continuous features, the gradient is actually an analytic linear
approximation of the classifier, providing that the classifier is differentiable. 
While there are good reasons to replace the gradient when the
classifier is non-differentiable or features are discrete, 
it seems wasteful not to use it for differentiable classifiers 
over continuous features.
In the latter setting, alternative attribution methods may be just a poor substitute
for the gradient
\FT{, not only giving} inaccurate explanations, but 
also
unnecessarily difficult to compute.
For example, \cite{garreau2020explaining} 
showed that when the prediction model to be explained by LIME is linear, 
the expected coefficients of the approximating linear model are proportional 
to the partial derivatives of the prediction model.
Since the partial derivatives exactly capture the behaviour of a linear
function, this shows that LIME is to some extent faithful to linear models.
However, the authors also found that the expected error of the linear approximation is bounded away from zero. Furthermore, as the approximation
has to be computed based on perturbations of the input, the scores are noisy
(as they depend on the sampled neighbours) and, compared to the gradient, relatively expensive to compute. More recently, \cite{Agarwal21} showed
that some natural configurations of LIME converge to the same scores in expectation as a smoothed
version of the gradient. While this can be desirable in the discrete
setting, the original gradient seems to be a more accurate and more efficient
explanation of the classifier's true behaviour in the continuous setting.

Since Gradient-scoring has been proposed as an explanation method in \cite{baehrens2010explain},
many authors noted lack of robustness in the sense that the scores 
assigned to features can change significantly when moving to 
"neighbouring" points \cite{heo2019fooling,ghorbani2019interpretation,wang2020gradient}.
In order to improve robustness, it has been suggested to smoothen
the gradient, for example, by integrating gradients from a reference point \cite{sundararajan2017axiomatic} or by averaging the gradients at neighbouring
points \cite{ancona2018towards}. However, it seems that the
robustness problems have been mainly observed in settings with discrete data,
where neighbours are generated by masking discrete features like words or
pixels. It is well known from the literature on adversarial machine learning
that even conceptually continuous models like neural networks can behave
discontinuously in these settings \cite{szegedy2014intriguing}. 
Therefore, it is not surprising that the gradients change significantly
when taking discrete steps. We should actually demand that a faithful
explanation method reflects this discontinuity if it is supposed to explain
what the model does and not just what the user expects it to do.
Only when the prediction model itself is 'robust'
should we demand that the scores are 'robust'.




   
\section{Preliminaries}

The abstract goal of classification is to map 
\FT{inputs} $\cinput$ to outputs $\coutput$.
We think of the inputs as feature 
\FT{vectors}
$\cinput = (x_1, \dots, x_k)$, where the i-th value is taken from some continuous 
domain $D_i \subseteq \mathbb{R}$.
We let $\dom = \bigtimes_{i=1}^k D_i$ denote the cartesian product of the individual
domains.
Given an input $\cinput = (x_1, \dots, x_k)$,
with a slight abuse of notation, we let
$(\cinput_{-i},x) = (x_1, \dots, x_{i-1}, x, x_{i+1}, \dots, x_k)$
denote the input where the i-th component has been replaced
with $x$. 
The output $\coutput$ is taken from a finite set  $L$ of class labels.
A \emph{classification problem} $P = ((D_1, \dots, D_k), L, E)$ consists of the domains, the class labels and a set of training examples $E = \{(\cinput_i, \coutput_i) \mid 1 \leq i \leq N, \cinput_i \in \dom, \coutput_i \in L\}$. 
The examples are used to train a classifier. We will
not be concerned with training 
\FT{and} assume that a classifier
is given. 

A \emph{probabilistic classifier} is a function 
$\classifier: \dom \times L \rightarrow [0,1]$
that assigns a probability  $\classifier(\cinput, \coutput)$ to every pair $(\cinput, \coutput)$ 
such that $\sum_{l \in L} \classifier(\cinput, l) = 1$. 
$\classifier(\cinput, \coutput) \in [0,1]$ can be understood as the confidence of the classifier that an example with features $\cinput$ belongs to 
the class $\coutput$. 
A classification decision can be made,
for example, by picking the label with the highest probability
or by defining a threshold value for the probability.
To simplify notation, we will often write $\classifier_{\coutput}(\cinput)$ instead of $\classifier(\cinput, \coutput)$ in the 
\FT{remainder}.

As we point out later, some strong notions of faithfulness
can only be satisfied by \emph{monotonic} classifiers.
We call a classifier \emph{monotonically increasing (resp. decreasing) wrt. the label $y$ and the $i$-th feature}
iff for all inputs $\cinput \in \dom$,
and $x_i' \in D_i$, $x_i < x_i'$ implies that
$\classifier_y(\cinput) \leq \classifier_y((\cinput_{-i},x_i'))$
(resp. $\classifier_y(\cinput) \geq \classifier_y((\cinput_{-i},x))$)
and 
$x_i > x_i'$ implies that
$\classifier_y(\cinput) \geq \classifier_y((\cinput_{-i},x_i'))$
(resp. $\classifier_y(\cinput) \leq \classifier_y((\cinput_{-i},x))$).
If 
\FT{these constraints hold with strict inequality}, we call the classifier 
\emph{strictly monotonically increasing (resp. decreasing) wrt. \FT{$y$ and} the $i$-th feature}.
We call a classifier \emph{(strictly) monotonic} if for every label and for every feature, the classifier is (strictly) monotonically increasing or (strictly) monotonically decreasing wrt. \FT{the label and} the feature. 


Given a classification problem $P = ((D_1, \dots, D_k), L, E)$,
an input $\cinput \in\dom$,
a class label $l \in L$
and a probabilistic classifier $\classifier$ for the problem,
an \emph{attribution method} is a function
$\attributionM: \dom \rightarrow \mathbb{R}^k$
that assigns a scoring vector to every input of the classifier.
Intuitively, the i-th component $\attributionM(\cinput)_i$ 
should \FT{``}reflect\FT{''} the influence of the
i-th feature on the probability of $l$ under 
$\classifier$ at the input $\cinput$.
If this is the case, we 
may deem the attribution method to be
\emph{faithful to the classifier}.
However, while intuitive, this notion is rather vague.
To obtain a formally testable criterion for faithfulness, we 
need a formal
definition that can be checked for different attribution methods.

In the remainder, we will assume as given some probabilistic classifier $\classifier$
for some classification problem $P$.

\section{Qualitative Faithfulness }
\label{sec_qualitative_faithfulness}

\subsection{Local and Strong Faithfulness}

A first desirable notion of an attribution method 
is \emph{qualitative faithfulness}:
if the score $\attributionM(\cinput)_i$ of the i-th feature is positive/negative, 
then the classifier should be monotonically increasing/ decreasing
wrt. the feature. 
However, many classifiers
are only locally monotonic, so that a simple score cannot represent the
actual global behaviour of the model.
We therefore distinguish between \emph{local} and \emph{global} faithfulness.
\begin{definition}[Qualitative Faithfulness]
\label{def_qual_faithful}
An attribution method $\attributionM: {1,\dots, k} \rightarrow \mathbb{R}$ is called \emph{locally qualitatively faithful to $\classifier$}
if 
for all inputs $\cinput \in \dom$ and 
for all labels $l \in L$
there exists 
$\epsilon > 0$ such that
for all 
$i=1,\dots,k$ and
for all 
$x \in D_i\setminus \{x_i\}$
such that $|x - x_i| < \epsilon$,
we have 
\begin{enumerate}
    \item if $\attributionM(\cinput)_i >0$, then 
    $\classifier_l(\cinput) < \classifier_l(\cinput_{-i},x)$
     if $x>x_i$ and
     $\classifier_l(\cinput) > \classifier_l(\cinput_{-i},x)$
     if $x<x_i$\FT{;} and
     \item if $\attributionM(\cinput)_i <0$, then
     $\classifier_l(\cinput) > \classifier_l(\cinput_{-i},x)$
     if $x>x_i$ and
      $\classifier_l(\cinput) < \classifier_l(\cinput_{-i},x)$
     if $x<x_i$.
\end{enumerate}
If  
$\attributionM: {1,\dots, k} \rightarrow \mathbb{R}$ is 
locally qualitatively faithful to $\classifier$ for all
inputs $\cinput$ and all $\epsilon >0$,
we call it \emph{(globally) qualitatively faithful to $\classifier$}.
\end{definition}
Intuitively, local qualitative faithfulness demands that the scores
reflect the true qualitative influence of a feature in a region
close to the given input. That is, when the score is positive (negative), then
the probability should increase (decrease) when increasing  the feature.
If this requirement holds on the whole domain, the attribution method
is called globally qualitatively faithful.

One may consider a stronger notion of faithfulness that also demands that the 
classifier ignores the attribute when the score is $0$. Let us note that
this definition is only meaningful when we assume that 
\FT{the classifier is}
increasing, decreasing or ignorant 
\FT{wrt.} every feature.
For non-monotonic classifiers that can change their 
behaviour from increasing to decreasing, this requirement
is impossible to satisfy.
However, as the property is desirable for monotonic classifiers, 
we also consider it here and call it \emph{strong} faithfulness.
\begin{definition}[Strong Qualitative Faithfulness]
An  attribution method $\attributionM: {1,\dots, k} \rightarrow \mathbb{R}$ is called \emph{strongly (locally) qualitatively faithful to $\classifier$}
if 
1) it is (locally) qualitative faithful and
2) for all inputs $\cinput \in \dom$ and 
for all labels $l \in L$
there exists 
$\epsilon > 0$ such that
for all 
$i=1,\dots,k$ and
for all 
$x \in D_i\setminus \{x_i\}$
such that $|x - x_i| < \epsilon$,
we have
$\classifier_l(\cinput) = \classifier_l(\cinput_{-i},x)$
whenever $\attributionM(\cinput)_i =0$.\\
If  
$\attributionM: {1,\dots, k} \rightarrow \mathbb{R}$ is 
locally strongly qualitatively faithful to $\classifier$ for all
inputs $\cinput$ and all $\epsilon >0$,
we call it \emph{(globally) strongly qualitatively faithful to $\classifier$}.
\end{definition}
Intuitively, we call an attribution method locally (strongly) qualitatively
faithful if the scores represent the true qualitative effect
(positive/(neutral)/negative) in an $\epsilon$-environment of 
input\FT{s}. Note that we allow that the size of the environment
depends on the input.
A globally qualitatively faithful attribution method
computes scores that 
represent the true effect over the whole domain.
However, as we illustrate next, 
some of these desiderata cannot be satisfied 
if the classifier 
\FT{to be explained} is non-monotonic.

\subsection{Feasibility of Faithfulness}

Qualitative faithfulness is a natural property for attribution methods
\FT{, but} to which extent can 
\FT{it} be satisfied at all?
\FT{We explore this question in 
the next example.} 
\begin{example}
\label{exp_nonlinear_classifier}
Consider a simple binary classification
problem over a single feature with domain $D = \mathbb{R}$ 
\FT{where, intuitively,} an 
\FT{input} should be
classified as positive if the value of the feature is
``sufficiently far away from $0$". 
A typical example is anomaly detection, where
\FT{the feature} corresponds to the deviation of an observation from the mean or median
and the decision
threshold is
based on the variance or 
interquartile range.
Consider the classifier $\classifier_1(x) = \phi_l(x^2)$,
where $\phi_l(z) = \frac{1}{1+ \exp(-z)}$
is the logistic function.
We could classify an input $x$ as positive if
$\classifier_1(x) > 0.7$.
Figure \ref{fig:fig_non_mon_exp} shows the graphs of the functions $\phi_l(x)$ and $\classifier_1(x)$.
Intuitively, $\phi_l(
\FT{x})$ squashes its input between $0$ and $1$. 
While $\phi_l$ is 
monotonically increasing,
our classifier $\classifier_1(x) = \phi_l(x^2)$ is monotonically
decreasing for $x<0$ and monotonically increasing for $x>0$.
In particular, at $x=0$, it is decreasing in one direction and increasing in
the other.
Now consider the input $x=-0.5$. We have 
$\classifier_1(-0.5) = \phi_l(0.25) \approx 0.56$ \FT{(red dot in Figure \ref{fig:fig_non_mon_exp})}.
When we increase $x$ to $0$, we have 
$\classifier_1(0) = \phi_l(0) = 0
\FT{.}5$ \FT{(green dot)}.
Since the probability decreased, the score under a strongly faithful attribution method must be negative.
However, as we increase $x$ further to 
$1$, we have 
$\classifier_1(1) = \phi_l(1) \approx 0.73$ \FT{(purple dot)}.
Since the probability increased, the score under a strongly 
faithful attribution method must also be positive.
This is clearly impossible and, therefore, there can be no
\emph{globally qualitatively faithful} attribution method for $\classifier_1$.
Also note that an attribution method can only be
\emph{qualitatively faithful} in an environment that depends on the input.
To see this, consider an arbitrary $\epsilon > 0$
and the input $-\frac{\epsilon}{2}$.
Then we have $\classifier_1(-\frac{\epsilon}{2}) > \classifier_1(-\frac{\epsilon}{4})$, but
$\classifier_1(-\frac{\epsilon}{2}) < \classifier_1(\epsilon)$.
Intuitively, the closer the input is to $0$, the smaller is the
environment in which an attribution method can be
qualitatively faithful.
Finally, note that there can be no \emph{strongly qualitative\FT{ly faithful}} explanation
for $\classifier_1$ at $0$ because $\classifier_1$ is neither increasing, nor decreasing,
nor does it ignore the input at this point. 
\end{example}
\begin{figure}[tb]
	\centering
		\includegraphics[width=\columnwidth]{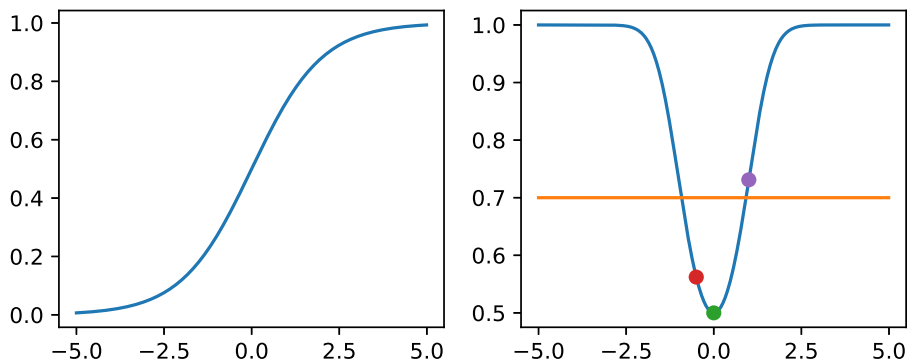}
	\caption{Logistic function and non-monotonic classifier \FT{(for Example~\ref{exp_nonlinear_classifier})}.}
	\label{fig:fig_non_mon_exp}
\end{figure}
As the previous example illustrates, the positive/negative effect of a 
feature can only be determined locally in an environment that depends on the
input. 
In our previous example,
the closer the negative input is to $0$, the smaller is the environment in which we can expect scores that are
qualitatively faithful to the model. 
This is because the effect switches from
negative to positive at $0$.
Such a behaviour \FT{may emerge} in many domains. For example,
when predicting health risks, the probability often increases
when particular health markers deviate substantially from a 
default value 
\FT{ (for} instance, both underweight and overweight 
and both low and high blood pressure could be seen as red flags\FT{)}.
\FT{Furthermore, }
multiple changes in the behaviour of a classifier can naturally occur 
for geographic features
\FT{, e.g.} latitude and longitude in real estate datasets\FT{: here, when predicting property demand,} assuming monotonic
behaviour is unrealistic 
because 
\FT{the popularity of neighbouring areas may be unrelated}.

\subsection{Faithfulness and Gradients}

We will now argue that faithfulness is closely related to gradients
when we try to explain differentiable classifiers over continuous data. 
Let us first note that the local behaviour of a function can be precisely
captured by its partial derivatives. Since many probabilistic classifiers like logistic regression and many neural networks correspond
to differentiable functions, it seems worthwhile to work towards a
characterization of faithfulness in terms of differentiability. 
Following \cite{nocedal2006numerical},
we call a function $f: D \rightarrow \mathbb{R}$ with domain
$D \subseteq \mathbb{R}^n$ 
differentiable at $\vecx \in D$ iff
there is a vector $g \in \mathbb{R}^n$ such that
$\lim_{\vecy \rightarrow 0} \frac{f(\vecx+\vecy) - f(\vecx) - g^{\top} \vecy}{\|\vecy \|}=0$.
Intuitively, this means that $f(\vecx+\vecy) \approx f(\vecx) +  g^{\top} \vecy$, 
that is, $f$ can be approximated locally by a linear function. 
The vector $g$ is called the gradient of $f$ at $\vecx$ and denoted by 
$\nabla f(\vecx)$. Its i-th component $g_i$ is denoted by
$\frac{\partial f}{\partial x_i}(\vecx)$ and is called the $i$-th partial derivative
of $f$ at $\vecx$. 
Intuitively, $\frac{\partial f}{\partial x_i}(\vecx)$ measures
the growth of $f$ at $\vecx$ along the i-th dimension.
In the following proofs, we will often use the fact that the partial derivative of 
a function $f$ can be written as $\frac{\partial f}{\partial x_i}(\vecx) =
 \lim_{h \rightarrow 0} \frac{f(\vecx + h e_i) -f(\vecx)}{h}$,
 where $e_i$ denotes the i-th unit vector \cite{nocedal2006numerical}.

We first note that the partial derivatives of features give us
a locally qualitative faithful attribution method for every
differentiable probabilistic classifier.
We call the corresponding attribution method \emph{Gradient-scoring}.
\begin{definition}[Gradient-scoring]
The attribution method
$\attributionMgrad$ 
defined by $\attributionMgrad(\cinput) = \nabla \classifier_l(\cinput)$
is called \emph{Gradient-scoring}.
\end{definition}
\begin{proposition}
\label{prop_grad_scor_faithful}
If $\classifier$ is \emph{differentiable},
then Gradient-scoring  is 
locally qualitatively faithful to $\classifier$.
\end{proposition}
\begin{proof}
We have to check that 
the two cases in the definition
of a local qualitative faithful attribution method are satisfied
for $\attributionMgrad$.
Assume first that the score of the i-th feature for label $l$ is positive, that is,
$0 < \attributionMgrad(\cinput)_i 
= \frac{\partial \classifier_l}{\partial x_i}(\cinput) 
= \lim_{h \rightarrow 0} \frac{\classifier_l((\cinput_{-i},x_i+h)) -\classifier_l(\cinput)}{h}$.
 For the case $x > x_i$,
 let $h = x - x_i > 0$. Then we have 
 $\classifier_l(\cinput_{-i},x) -\classifier_l(\cinput)
 = \classifier_l(\cinput_{-i},x_i+h) -\classifier_l(\cinput)
 = h \cdot \frac{\classifier_l(\cinput_{-i},x_i+h) -\classifier_l(\cinput)}{h}
 \approx h \cdot \frac{\partial \classifier_l}{\partial x_i}(\cinput)
 > 0$ if $x$ is sufficiently close to $x_i$.
 Hence, there exists an $\epsilon >0$ such that 
 $\classifier_l(\cinput) < \classifier_l(\cinput_{-i},x)$
 if $|x - x_i| < \epsilon$ as desired.
 The case $x < x_i$ can be checked symmetrically.
 The second case for 
 $\attributionMgrad(\cinput)_i < 0$
follows analogously.
\end{proof}
We illustrate the desirable behaviour of 
\FT{Gradient-scoring} in two simple examples.
\begin{example}
Consider again the non-linear classifier $\classifier_1(x) = \phi_l(x^2)$ 
from Example \ref{exp_nonlinear_classifier}.
As we noted 
\FT{there}, there
can be no strongly faithful attribution method for $\classifier_1$,
so a locally faithful explanation is the best that we can hope for.
It is well known that the derivative of the logistic function is
$\phi_l'(x) = \frac{\exp(x)}{(1+\exp(x))^2}$. Note, in particular, 
that $\phi_l'(x) > 0$ for all $x \in \mathbb{R}$ because both the 
numerator and denominator are necessarily positive. The derivative of our
classifier is $\classifier_1'(x) = 2x \cdot \phi_l'(x^2)$,
which is negative for $x<0$ and positive for $x>0$.
By looking at Figure \ref{fig:fig_non_mon_exp}, we can see that
the derivative is indeed locally qualitatively faithful to the classifier,
as it is mononotonically decreasing for $x<0$ and monotonically increasing
for $x>0$. This is not a coincidence, of course, because the derivative captures exactly the local change of a function.
\end{example}
\begin{example}
\label{exp_log_regression_grad_scores}
As another example, let us consider a linear classifier.
A logistic regression classifier for a binary classification problem
has the form
$\classifier(\cinput) = \phi_l(\sum_{i=1}^n w_i \cdot \cinput_i)$,
where $w_i \in \mathbb{R}$.
The partial derivative wrt. the i-th feature is 
 $\frac{\partial \classifier}{\partial x_i}(\cinput)
 = w_i \cdot \phi_l'(\sum_{i=1}^n w_i \cdot \cinput_i)$.
 As before, the sign depends only on $w_i$. So the score is
 positive if $w_i > 0$, negative if $w_i < 0$ and
 \FT{0} if $w_i = 0$. It is clear from the definition of logistic
 regression that the weights reflect the true behaviour of the classifier, 
 so Gradient-scoring is indeed a qualitative faithful attribution method. 
In fact, it is even globally strongly qualitatively faithful. Of course, the qualitative influence on the model
 can be directly seen by looking at the coefficients, but logistic regression
 is a reasonable test for the reliability
 of an attribution method. We will use this 
 in Section \ref{sec_experiments} to evaluate the faithfulness of 
 \FT{various} attribution methods.
\end{example}
Let us note that every linear classifier is monotonic. If the coefficient
of the $i$-th feature is positive (negative), the classifier is monotonically
increasing (decreasing) wrt. the i-th feature. 
Let us note that
Gradient-scoring is globally qualitative faithful for every monotonic classifier.
\begin{proposition}
\label{prop_mon_grad_global_faithful}
If $\classifier$ is \emph{differentiable and monotonic},
then Gradient-scoring is
globally qualitatively faithful to $\classifier$.
\end{proposition}
\begin{proof}
We prove the claim by going through the two cases again.
If $\attributionMgrad(\cinput)_i >0$, then the partial derivative
is positive at this point. Hence, $\attributionMgrad(.)_i$
is strictly increasing at this point. 
Since $\attributionMgrad$ is monotonic,
$\attributionMgrad(.)_i$ must be increasing (not necessarily strictly increasing) 
on the whole domain of 
the $i$-th feature. 
If $x>x_i$, there must be a $y \in (x_i, x)$ such that
$\classifier_l(\cinput) < \classifier_l(\cinput_{-i},y) \leq \classifier_l(\cinput_{-i},x)$
where the strict inequality follows from strict monotonicity at $\cinput$ 
and the second inequality from monotonicity on the whole domain.
Symmetrically, one can check that
$\classifier_l(\cinput) > \classifier_l(\cinput_{-i},x)$
for all $x<x_i$.

The second case from the definition of qualitative faithfulness can be checked symmetrically.
\end{proof}
Let us note that the result is actually slightly stronger: If the classifier
is monotonic wrt. a particular feature, then Gradient-scoring is globally qualitatively faithful. If it is strictly monotonic, then Gradient-scoring is globally strongly qualitatively faithful.

In summary, Gradient-scoring is locally qualitatively faithful for every
differentiable classifier\FT{;}
it is globally qualitatively faithful for monotonic differentiable classifiers. Thus, given that 
there can be no 
attribution method that is globally qualitatively faithful for non-monotonic classifiers \FT{(as illustrated in Example  \ref{exp_nonlinear_classifier})} 
Gradient-scoring gives us the strongest guarantees that 
we can hope for.

There can be other locally qualitatively faithful attribution methods
for differentiable probabilistic classifiers, of course.
However, as we explain next, they all must be closely related to Gradient-scoring.
Intuitively, the following proposition states that an
attribution method for a differentiable
classifier can only be locally qualitatively faithful 
if the sign of its score equals the sign
of the partial derivative. In other words,
\emph{every qualitatively faithful attribution method must be qualitatively equivalent to Gradient-scoring}.
\begin{proposition}
\label{prop_qual_faithfulness_and_grad}
If $\classifier$ is differentiable and there exists an attribution method $\attributionM: {1,\dots, k} \rightarrow \mathbb{R}$ that is locally qualitatively faithful to $\classifier$, 
then for all $i=1,\dots,k$, for all inputs $\cinput \in \dom$
and all labels $l \in L$, we have  
\begin{enumerate}
    \item if $\attributionM(\cinput)_i >0$, then
     $\attributionMgrad(\cinput)_i \geq 0$, and
     \item if $\attributionM(\cinput)_i <0$, then
     $\attributionMgrad(\cinput)_i \leq 0$.
\end{enumerate}
\end{proposition}
\begin{proof}
Consider the case $\attributionM(\cinput)_i >0$.
Since $\attributionM$ is locally qualitatively faithful,
we have that, for all $h$ that are sufficiently close to $0$, $\classifier_l(\cinput_{-i},x_i+h) -\classifier_l(\cinput) > 0$
if $h>0$ and $\classifier_l(\cinput_{-i},x_i+h) -\classifier_l(\cinput) < 0$
if $h<0$.
Therefore, 
$\frac{\classifier_l(\cinput_{-i},x_i+h) -\classifier_l(\cinput)}{h} > 0$
and
$\frac{\partial \classifier_l}{\partial x_i}(\cinput) =
 \lim_{h \rightarrow 0} \frac{\classifier_l(\cinput_{-i},x_i+h) -\classifier_l(\cinput)}{h} \geq 0$.

The second item follows symmetrically.
\end{proof}


\subsection{Summary}

We quickly summarize the main points of this section.
An attribution method should be qualitatively faithful.
However, as demonstrated in  Example \ref{exp_nonlinear_classifier},
we cannot guarantee global or strong
faithfulness without making
additional assumptions about the nature of the classifier that we want to
explain. For non-monotonic classifiers, attributive explanations
typically can\FT{not}  
\FT{exhibit strong qualitative faithfulness}.
We showed that Gradient-scoring is always locally qualitatively faithful
(Proposition \ref{prop_grad_scor_faithful})
and that every locally qualitatively faithful attribution method must
be qualitatively equivalent to Gradient-scoring (Proposition \ref{prop_qual_faithfulness_and_grad}).
Stronger guarantees can only be given for monotonic classifiers and
Gradient-scoring guarantees global qualitative faithfulness for these classifiers (Proposition \ref{prop_mon_grad_global_faithful}).
As we will demonstrate in Section \ref{sec_experiments}, this is not the
case for some other popular attribution methods.
Let us also note that monotonicity alone is not sufficient to allow for 
strongly faithful explanations. For example, a monotonically increasing 
function may be constant for a while before it continues increasing.
To guarantee strongly \FT{qualitatively} faithful explanations, we have to assume that
a classifier is, for every feature, either strictly monotonic or constant.
In this case, every locally qualitative faithful attribution method
is also strongly qualitatively faithful.

\section{Quantitative Faithfulness}
\label{sec_quantitative_faithfulness}

In the previous section, we showed that every qualitatively faithful
attribution method must be qualitatively equivalent to Gradient-scoring.
However, usually, we want that the attribution method is also quantitatively
faithful in the sense that a larger score reflects a larger influence of
a feature on the prediction. This quantitative faithfulness could be
formalized in different ways. One intuitive way is to demand that the change
in output of the prediction model when increasing a feature is proportional to its score. 
Formally, 
\begin{equation}
\label{eq_qfaith_approx_comp}
 \classifier_l(\cinput_{-i},x_i+h) \approx  \classifier_l(\cinput) + h \cdot \attributionM(\cinput)_i.   
\end{equation}
More generally, we can demand
\begin{equation}
\label{eq_qfaith_approx}
 \classifier_l(\cinput + \vech) \approx   
 \classifier_l(\cinput) + \vech^{\top} \attributionM(\cinput)
 = \classifier_l(\cinput) +
 \sum_{i=1}^k h_i \cdot \attributionM(\cinput)_i,   
\end{equation}
where equation \eqref{eq_qfaith_approx_comp} is recovered when we let $\vech$
be the $i$-th unit vector.
Let us note that both conditions are trivially satisfied for small $h$ in our
setting. This is because our classifier is differentiable and therefore continuous by assumption. Hence, whenver $\vech \approx 0$, we also
have $ \classifier_l(\cinput + \vech) \approx  \classifier_l(\cinput)$
(by continuity)
and $\vech^{\top} \attributionM(\cinput) \approx 0$
(because $\attributionM(\cinput)$ is constant).
To make the condition non-trivial, we have to demand that the approximation
error goes faster to $0$ than $\vech$.
We let $e(\vech) = \classifier_l(\cinput + \vech) - \classifier_l(\cinput) - \vech^{\top} \attributionM(\cinput)$ denote 
the error of the approximation when changing the input by $\vech$. 
Then we can rewrite \eqref{eq_qfaith_approx} more precisely as
\begin{equation}
\label{eq_qfaith_exact}
\classifier_l(\cinput + \vech) =  \classifier_l(\cinput) + \vech^{\top} \attributionM(\cinput) +  e(\vech).
\end{equation}
To get a non-trivial condition, we require that $e(\vech)$ goes 
significantly faster to $0$ than $\vech$, that is, we additionally demand that 
$\lim_{\vech \rightarrow 0} \frac{e(h)}{\|\vech\|} = 0$. 
\begin{definition}[Local Quantitative Faithfulness]
\label{def_quant_faithfulness}
An attribution method $\attributionM: {1,\dots, k} \rightarrow \mathbb{R}$ is called \emph{locally quantitatively faithful to $\classifier$}
if 
for all inputs $\cinput \in \dom$, 
for all labels $l \in L$,
for all 
$i=1,\dots,k$ and
for all $\vech$ such that $\cinput + \vech \in \dom$,
we have
$
\classifier_l(\cinput + \vech) = \classifier_l(\cinput) +\vech^{\top} \attributionM(\cinput) +  e(\vech),
$
where $e(\vech)$ is an error term with 
$\lim_{\vech \rightarrow 0} \frac{e(h)}{\|\vech\|} = 0$.
\end{definition}
Gradient-scores satisfy local quantitative faithfulness by definition.
More interestingly, the property characterizes Gradient-scoring in the 
sense that it is the only attribution method that does so.
\begin{proposition} 
\label{thm_quant_faithfulness_characterization}
If $\classifier$ is differentiable, then
$\attributionM$ is 
locally quantitative faithful to $\classifier$
if and only if $\attributionM = \attributionMgrad$.
\end{proposition}
\begin{proof}
To see that $\attributionMgrad$ satisfies local quantitative faithfulness, note that
$\lim_{\vech \rightarrow 0} \frac{e(h)}{\|\vech\|} = 
\lim_{\vech \rightarrow 0} \frac{\classifier_l(\cinput + \vech) - \classifier_l(\cinput) - \vech^{\top} \attributionMgrad(\cinput)}{\|\vech\|}
= 0
$ because $\classifier_l$ is differentiable (see Section 4.3 to recall the definition of differentiability).

To show that Gradient-scoring is the only attribution method
that satisfies this property, consider an arbitrary attribution
method $\attributionM$ that satisfies local quantitative faithfulness.
We show that it is equal to Gradient-scoring by showing that
for all inputs $\cinput$ and for all feature dimensions $i$,
we have
$|\attributionMgrad(\cinput)_i - \attributionM(\cinput)_i| < \epsilon$
for every $\epsilon > 0$. Since $\epsilon$ can be arbitrarily small,
we must have $\attributionMgrad(\cinput)_i = \attributionM(\cinput)_i$,
which proves the claim.
Let $e_{g}$ and $e_{s}$ denote the error functions for
$\attributionMgrad$ and $\attributionM$, respectively.
Then we can find an $h>0$ such that for all unit vectors $u_i$, we have 
$\big| \frac{e_{g}(h \cdot u_i)}{h} \big|  < \frac{\epsilon}{2}$ and
$\big| \frac{e_{s}(h \cdot u_i)}{h} \big| < \frac{\epsilon}{2}$.
Then 
\begin{align*}
 |\attributionMgrad(\cinput)_i - \attributionM(\cinput)_i|
&= \big| \frac{h\cdot \attributionMgrad(\cinput)_i - h\cdot  \attributionM(\cinput)_i}{h}\big|  \\
&= \big| \frac{h \cdot \attributionMgrad(\cinput)_i - 
\classifier_l(\cinput + h \cdot u_i) + \classifier_l(\cinput)}{h} \\
&\quad + \frac{\classifier_l(\cinput + h \cdot u_i) - \classifier_l(\cinput)
- \attributionM(\cinput)_i) \cdot h} 
{h}\big|    \\
&\leq
\big| \frac{e_{g}(h \cdot u_i)}{h}\big|  + \big| \frac{e_{s}(h \cdot u_i)}{h}\big|  < \epsilon,
\end{align*}
which completes the proof.
\end{proof}
Intuitively, if any other attribution method can guarantee
that the scores are proportional to the actual influence of the features
(it satisfies equation \eqref{eq_qfaith_approx}),
then the method is either equivalent to gradient scoring or it 
is based on an inferior approximation of the model, that is, 
it does not satisfy $\lim_{h \rightarrow 0} \frac{e(h)}{\|\vech\|} = 0$.
In more precise terms, we can state that the approximation error with
respect to $\attributionMgrad$ goes significantly faster to $0$
than the error with respect to every other attribution method.
\begin{proposition}
\label{prop_quant_faithfulness_convergence_rate}
If $\classifier$ is differentiable, then for every attribution
method $\attributionM \neq \attributionMgrad$, we have 
$\lim_{\vech \rightarrow 0} \frac{e_g(h)}{e_s(h)} = 0$, 
where  $e_{g}$ and $e_{s}$ denote the error functions for
$\attributionMgrad$ and $\attributionM$, respectively.
\end{proposition}
\begin{proof}
Since $\attributionM \neq \attributionMgrad$, Proposition 
\ref{thm_quant_faithfulness_characterization} implies 
$\lim_{\vech \rightarrow 0} \frac{e_s(h)}{\|\vech\|} \neq 0$.
Hence, we can find an $\epsilon > 0$ such that $\lim_{\vech \rightarrow 0} \big| \frac{e_s(h)}{\|\vech\|}\big| \geq \epsilon > 0$. Then we 
have
$\lim_{\vech \rightarrow 0} \frac{e_g(h)}{e_s(h)}
=\lim_{\vech \rightarrow 0} \frac{\frac{e_g(h)}{\|h\|}}{\frac{e_s(h)}{\|h\|}}
\leq \frac{0}{\epsilon} = 0.
$
\end{proof}

Similar to qualitative faithfulness, we usually cannot expect to find a 
"globally" quantitatively faithful attribution method. 
Let us note that while monotonicity of a classifier allows global qualitative faithfulness, it is not sufficient to allow "global" quantitative faithfulness.
This is because
a monotonic function can still increase (or decrease) faster and slower
in different regions of its domain. Conceptually, stronger guarantees
can be obtained by moving from just the gradient
(a first-order Taylor approximation) to higher-order Taylor approximations
of the classifier. For example, scores for the influence of pairs of
features can be obtained from the Hessian matrix. 
However, this would also increase the
computational cost considerably and the explanation would be less succinct 
(but more accurate).

\section{Experiments}
\label{sec_experiments}

In this section, we 
investigate to which extent different attribution methods \FT{(see Section~\ref{sec:expl})}
are faithful to a classifier and how close they are to each other on three
continuous datasets (see Section~\ref{sec:data}).
To evaluate faithfulness, we consider Logistic Regression classifiers,
where the sign of coefficients represent the true qualitative effect of 
features on models and their magnitude their impact (see Section~\ref{sec:LR}).
To also compare attribution methods for non-linear models, we consider Multilayer Perceptrons  with unknown ground truth (see Section~\ref{sec:MLP}).
The code for all experiments is available at

\url{https://github.com/XiangYin2021/Revisit-Faithfulness}

\subsection{Datasets}
\label{sec:data}
\FT{We use}
three binary classification problems with continuous features from the UCI machine learning repository\footnote{\url{http://archive.ics.uci.edu/ml/index.php}}.  

 \emph{BNAT}: The Banknote Authentication data set consists of image information features (like entropy of image, or variance of Wavelet Transformed image) of genuine and forged banknotes. The label identifies whether a banknote is genuine
 . 
 
 \emph{WDBC}: The Breast Cancer Wisconsin (Diagnostic) data set \cite{street1993nuclear} contains features of cell nuclei of breast mass images, and their corresponding diagnosis (Benign or Malignant) as labels. 

 \emph{PIMA}: The PIMA Indians Diabetes data set \cite{alcala2011keel} contains physical indicator features of females and their diagnosis of diabetes as label. 

Table \ref{dataset_info} summarizes some basic statistics of the data sets.
\begin{table}[t]
\caption{Basic data set statistics.}
\label{dataset_info}
\vskip 0.15in
\begin{center}
\begin{small}
\begin{sc}
\begin{tabular}{lrrrr}
\toprule
Dataset & Instances & Features & Label \\
\midrule
BKNT    & 1372 & 5  & Binary \\
WDBC    & 569  & 30 & Binary \\
PIMA    & 768  & 8  & Binary \\
\bottomrule
\end{tabular}
\end{sc}
\end{small}
\end{center}
\vskip -0.1in
\end{table}

\subsection{Attribution Methods}
\label{sec:expl}
We compare Gradient-scoring to three popular attribution
methods. 

\emph{LIME} \cite{ribeiro2016should} is a model-agnostic local explanation method.
In order to explain the output at one data point, LIME first perturbs this point to generate several points in a small neighborhood, and weighs the perturbed points based on their distance to the original point. LIME then trains a linear model to explain the original point by the weights of the 
linear model. We used the LIME implementation from \url{https://github.com/marcotcr/lime}
and applied the LIMETabularExplainer.

\emph{SHAP} \cite{lundberg2018explainable} uses the Shapley value from coalitional game
theory to compute the contribution of each feature to the prediction.
As the Shapley value is a discrete concept, SHAP relies heavily on sampling when being applied  to continuous features. 
We used the SHAP implementation from
\url{https://github.com/slundberg/shap}
and show results for the KernelExplainer.
As SHAP did not perform well in the experiments, we tried different
settings, but they did not improve the outcome.

\emph{SILO} 
\cite{bloniarz2016supervised} is a non-parametric regression method,
but can also be seen as an attribution method.
It first uses random forests to compute weights of data points in the
training set, and then selects supervised neighbourhoods to construct local linear models. The coefficients of the linear model can again be used
to assign scores to features.
\FT{We also mentioned MAPLE \cite{plumb2018model} before},
which extends SILO by a feature selection method. However, as we want to find out to which
extent different methods capture the true effects of all features, we will
consider only SILO in our experiments.
We use the MAPLE implementation
from \url{https://github.com/GDPlumb/MAPLE} and 
omit the feature selection step.
We configured it with 
n\_estimators=200; min\_sample\_leaf=10 and linear\_model=Ridge.

\subsection{Setup}

\subsubsection{Hardware} 
We ran all experiments on a MacBook Pro (OS: macOS Monterey 12.0.1; Processor: 2.3GHz, dual-core Intel Core i5; Memory: 8GB).

\subsubsection{Classifiers}

We used the default training settings of scikit-learn for Logistic Regression \FT{(LR)} and Multilayer Perceptrons \FT{(MLP)} with the following modifications:
\begin{itemize}
    \item{LR:} max\_iter=500. 
    \item{MLP:} We experimented with different numbers of hidden layers and neurons per layer and eventually used five hidden layers with eight neurons in each layer. The maximum number of 
    iterations (max\_iter) was set to 1000.
\end{itemize}
Table \ref{acc_f1_score} shows 
the performance of the classifiers on our datasets.
\begin{table}[t]
\caption{Accuracy and F1-score for Classifiers}
\label{acc_f1_score}
\vskip 0.15in
\begin{center}
\begin{small}
\begin{sc}
\begin{tabular}{crrr}
\toprule
Classifier & Dataset & Accuracy & F1-score \\

\midrule
        & PIMA    & 0.760  & 0.565  \\
    LR  & BKNT    & 0.964  & 0.958  \\
        & WDBC    & 0.956  & 0.937  \\
        
\midrule
          & PIMA    & 0.721  & 0.566  \\
    MLP   & BKNT    & 0.978  & 0.975  \\
          & WDBC     & 0.982  & 0.976  \\         
         
\bottomrule
\end{tabular}
\end{sc}
\end{small}
\end{center}
\vskip -0.1in
\end{table}

\subsection{
\FT{Experiments with} Logistic Regression \FT{(LR)}}
\label{sec:LR}

In our first set of experiments, we compute explanations for 
\FT{LR} because its feature weights can be seen as representing the
true effect of features on the LR classifier. 
One has to consider that the logistic function
is applied to the linear combination of features, so the feature weights
do not represent the true quantitative effect. However, as the logistic function
is monotonic, the weights show the true qualitative effect and the true order
of the impact of features. 

We used the 
\FT{LR}
implementation of scikit-learn\footnote{\url{https://scikit-learn.org/}} in our experiments. Gradient-scores can be easily computed analytically from the weights as explained in Example \ref{exp_log_regression_grad_scores}.

To evaluate how well different explanation
methods reflect the true effect of features on the classifier, we computed Spearman's rank correlation 
between the 
\FT{LR} weights and the attribution scores in Table \ref{LR_similarity}. 
We averaged \FT{the correlation} over all inputs from the test set
for each dataset.
Spearman's rank correlation 
measures the strength of the monotonic relationship between two variables.
As opposed to Pearson's correlation, it can also capture non-linear 
relationships and seems therefore well suited for the task. 
The gradient shows perfect
correlation with the weights.
This can also
be seen analytically because the gradient corresponds to the feature
weight multiplied by a constant that depends on the point (see Example \ref{exp_log_regression_grad_scores}). 
LIME is reasonably close for all datasets. SILO does well on
two datasets, but does not work well on WDBC. 
SHAP does not show any significant correlation with
the weights on any dataset.
\begin{table}[t]
\caption{Average Spearman rank correlation between 
\FT{LR} weights
and different explanation methods on test data.}
\label{LR_similarity}
\vskip 0.15in
\begin{center}
\begin{small}
\begin{sc}
\begin{tabular}{crrrrrr}
\toprule
Dataset     & GRAD  & SILO  & LIME  & SHAP \\
\midrule
PIMA  & \textbf{1.000} & 0.989  & 0.976  & 0.166  \\
BKNT  & \textbf{1.000} & 0.975  & \textbf{1.000}  & -0.073  \\
WDBC  & \textbf{1.000} & 0.378  & 0.986  & 0.272  \\

\bottomrule
\end{tabular}
\end{sc}
\end{small}
\end{center}
\vskip -0.1in
\end{table}

To give a better picture of the output of different attribution methods,
Tables \ref{LR_gradient_BKNT} - \ref{LR_gradient_WDBC}, 
show the 
\FT{LR}
weights
and the scores that have been assigned to features by the different
attribution methods for the first test data point of our three datasets\FT{, respectively}.
We ordered the features by  their logistic regression weights,
so that the qualitative score and their order can be compared easily.
\begin{table}[t]
\caption{
\FT{LR} weights and feature scores for the first test data point of the BNKT Dataset.}
\label{LR_gradient_BKNT}
\vskip 0.15in
\begin{center}
\begin{small}
\begin{sc}
\begin{tabular}{lrrrrr}
\toprule
          & Weight & Grad & SILO &LIME & SHAP \\
\midrule
vwti  & -10.240  & -0.536  & -0.748  & -0.270  & -0.443  \\
swti  & -7.508  & -0.393  & -0.571  & -0.212  & -0.014  \\
cwti  & -7.001  & -0.366  & -0.530  & -0.167  & -0.042  \\
entr  & 0.284  & 0.015  & 0.021  & 0.006  & 0.003  \\

\bottomrule
\end{tabular}
\end{sc}
\end{small}
\end{center}
\vskip -0.1in
\end{table}
\begin{table}[t]
\caption{
\FT{LR} weights and feature scores for the first test data point of the PIMA Dataset.}
\label{LR_gradient_PIMA}
\vskip 0.15in
\begin{center}
\begin{small}
\begin{sc}
\begin{tabular}{lrrrrr}
\toprule
 & Weight & Grad & SILO &LIME & SHAP \\
\midrule
bloo  & -0.734  & -0.139  & -0.146  & -0.022  & -0.011  \\
insu  & 0.021  & 0.004  & -0.004  & 0.000  & -0.000  \\
skin  & 0.581  & 0.110  & 0.118  & 0.018  & -0.002  \\
age   & 1.078  & 0.205  & 0.232  & 0.041  & 0.025  \\
diab  & 1.254  & 0.238  & 0.275  & 0.033  & 0.003  \\
preg  & 1.350  & 0.256  & 0.278  & 0.051  & 0.046  \\
bmi   & 2.855  & 0.542  & 0.611  & 0.063  & -0.006  \\
gluc  & 4.920  & 0.934  & 1.068  & 0.153  & 0.041  \\
\bottomrule
\end{tabular}
\end{sc}
\end{small}
\end{center}
\vskip -0.1in
\end{table}
\begin{table}[t]
\caption{Five lowest and highest 
\FT{LR} weights and corresponding
feature scores for the first test data point of the WDBC Dataset.}
\label{LR_gradient_WDBC}
\vskip 0.15in
\begin{center}
\begin{small}
\begin{sc}
\begin{tabular}{lrrrrr}
\toprule
      & Weight & GRAD & SILO &LIME & SHAP \\
\midrule
    mean fractal dimension & -0.821  & -0.168  & -0.046  & -0.021  & 0.016  \\
    se fractal dimension & -0.602  & -0.123  & -0.079  & -0.015  & 0.022  \\
    se compactness & -0.415  & -0.085  & 0.115  & -0.011  & 0.009  \\
    se symmetry & -0.330  & -0.068  & -0.008  & -0.007  & 0.006  \\
    se concavity & -0.143  & -0.029  & -0.062  & -0.001  & 0.001  \\
    se smoothness   & -0.015  & -0.003  & -0.058  & -0.002  & -0.000  \\
    se texture & -0.003  & -0.001  & 0.225  & 0.001  & -0.000  \\
    worst fractal dimension & 0.347  & 0.071  & 0.119  & 0.007  & -0.001  \\
    mean symmetry & 0.363  & 0.074  & 0.029  & 0.009  & -0.007  \\
    mean smoothness & 0.384  & 0.079  & -0.001  & 0.010  & -0.002  \\
    se concave points & 0.408  & 0.083  & -0.028  & 0.011  & -0.008  \\
    mean compactness & 0.571  & 0.117  & -0.311  & 0.016  & -0.016  \\
    se area & 0.677  & 0.139  & -0.078  & 0.011  & -0.002  \\
    se perimeter & 0.737  & 0.151  & -0.407  & 0.013  & -0.009  \\
    worst compactness & 0.799  & 0.163  & 0.176  & 0.023  & -0.009  \\
    se radius & 0.956  & 0.196  & 0.515  & 0.018  & -0.006  \\
    worst symmetry & 1.172  & 0.240  & 0.052  & 0.024  & 0.004  \\
    worst concavity & 1.387  & 0.284  & 0.027  & 0.042  & -0.004  \\
    mean concavity & 1.405  & 0.288  & 0.163  & 0.048  & -0.024  \\
    mean area & 1.446  & 0.296  & -0.953  & 0.042  & 0.003  \\
    mean texture & 1.452  & 0.297  & 0.262  & 0.038  & 0.038  \\
    worst smoothness & 1.472  & 0.301  & 0.266  & 0.041  & 0.005  \\
    worst area & 1.638  & 0.335  & 0.224  & 0.043  & 0.006  \\
    mean perimeter & 1.704  & 0.349  & 0.661  & 0.053  & -0.005  \\
    mean radius & 1.721  & 0.352  & 0.621  & 0.053  & -0.002  \\
    mean concave points & 1.851  & 0.379  & 0.079  & 0.064  & -0.009  \\
    worst texture & 2.069  & 0.423  & -0.212  & 0.063  & 0.060  \\
    worst perimeter & 2.077  & 0.425  & 0.884  & 0.064  & -0.004  \\
    worst radius & 2.260  & 0.462  & -0.583  & 0.073  & 0.004  \\
    worst concave points & 2.679  & 0.548  & 0.177  & 0.111  & 0.021  \\
\bottomrule
\end{tabular}
\end{sc}
\end{small}
\end{center}
\vskip -0.1in
\end{table}

Across all datasets, we can see that Gradient-scores always correctly capture the 
qualitative effect and the order of features as the theory suggests. 
Since SILO and LIME are based on linear approximations of the model,
we should expect a similar picture. However, as they depend on sampling,
the result may be more noisy. The experiments show indeed that both
capture the true effect of features relatively well. However, sometimes,
the order is not captured completely correctly and there can be problems
with the sign. For example, SILO
assigns a negative score to \emph{WORST TEXTURE} and \emph{WORST RADIUS} in Table \ref{LR_gradient_WDBC}
even though they have a large positive weight.
LIME works generally better, but does not capture the true
order perfectly. This can be seen, for example, for \emph{AGE}, \emph{DIAB} and
\emph{PREG} in
Table \ref{LR_gradient_PIMA}. However, while SILO and LIME have problems
distinguishing features that are close, they perform reasonably well overall.
The results for SHAP, in contrast, look quite random overall and 
often do not seem to reflect
the true effect of features on the classifier.

\subsection{
\FT{Experiments with} Multilayer Perceptron\FT{s (MLPs)}}
\label{sec:MLP}
\FT{LR} is well suited for experiments because the true effect
of features can be seen from the weights. However, in general, we are also
interested in explaining non-linear classifiers. Unfortunately, in
this case, an objective ground truth is rarely available. There has been some
work recently on creating synthetic prediction models with known ground
truth \cite{guidotti2021evaluating}, but the work does not cover
non-linear classifiers for continuous data (they only consider 
synthetic linear regression models for tabular data) and rely
on the assumption that the models correctly learnt
relationships from artificially generated data.
We feel that the gradient
does represent the ground truth because it exactly captures the local growth of
a function. This intuitive idea is made precise in Definition \ref{def_quant_faithfulness} and Proposition \ref{thm_quant_faithfulness_characterization}.
Let us note that the gradient is also used as the ground truth
for the synthetic linear regression models in \cite{guidotti2021evaluating}.
However, it is, of course, a philosophical question 
\FT{whether} this is what the
scores should represent. In lack of an objective ground truth,
we simply compare the Gradient-scores to the scores under other attribution
methods. While a low correlation is difficult to interpret, a high correlation
basically indicates that the methods give similar explanations to Gradient-scoring.

We used the 
\FT{MLP} implementation from scikit-learn for the experiments \FT{in this section}. 
For many libraries, gradients can be computed using their auto-differentiation 
functionality. However, we simply approximated the gradient
by the symmetric difference quotient $\tilde{g}(x) = \frac{f(x+\varepsilon)-f(x-\varepsilon )}{2\varepsilon}$ in our experiments.
Table \ref{MLP_similarity} shows the Spearman rank correlation between
the Gradient-score and the other methods. The correlation with LIME is high.
This can be expected analytically because LIME just computes another
linear approximation of the model. 
The correlation 
\FT{with} SILO is lower,
in particular, for WDBC. SHAP again does not show any correlation with
the Gradient-score. As no objective ground truth is available, we cannot 
objectively say that one method gives better explanations than another.
However, we feel that Gradient-scores are favorable because they have
a clearly defined  meaning
(they capture the local sensitivity of the model with respect to the feature)
and are not susceptible to noise caused by sampling.
In particular, it does not seem clear what SHAP scores actually
mean in a setting where features are continuous because the theory of Shapley
values has been developed in a discrete setting and SHAP scores rely heavily on approximation and sampling when features are continuous.

\begin{table}[t]
\caption{Average Spearman rank correlation between \FT{Gradient-scoring}
and other explanation methods on test data \FT{with MLPs}.}
\label{MLP_similarity}
\vskip 0.15in
\begin{center}
\begin{small}
\begin{sc}
\begin{tabular}{crrrrrr}
\toprule
Dataset & GRAD  & SILO  & LIME  & SHAP \\
\midrule
PIMA  & 1.000  & 0.868  & 0.905  & 0.256  \\
BKNT  & 1.000  & 0.640  & 0.837  & 0.260  \\
WDBC  & 1.000  & 0.293  & 0.938  & 0.297  \\

\bottomrule
\end{tabular}
\end{sc}
\end{small}
\end{center}
\vskip -0.1in
\end{table}

\subsection{Runtime Performance}

Overall, we feel that linear attribution methods like Gradient-scoring, 
LIME and SILO
are more natural in the continuous setting because they have a clear 
interpretation. LIME is highly correlated with Gradient-scoring,
so 
\FT{a natural question arises as to why one should prefer one method over the other.}
One consideration to take into account is runtime performance. 
\FT{Gradient-scores are} easy to compute in linear time with respect to the \FT{number of} features.
Ideally, this should be done using the auto-differentiation functionalities of libraries. When approximating
\FT{the gradient} using the symmetric difference quotient, one has
to consider the usual numerical approximation problems that can occur
when the classifier has large partial derivatives (e.g., when the neural network
has very large weights).
The runtime of linear explainers like LIME and SILO is also linear, but 
additionally depends on the number of samples that are taken and on
the training procedure for learning the local substitute model.
In particular, while a smaller sample size may result in faster runtime,
the scores may be too noisy if the sample size is chosen too 
small.
SHAP is usually most difficult to compute and relies heavily on
sampling. The complexity of computing SHAP scores has been analyzed
recently and is intractable in many interesting cases \cite{van2021tractability}.
As our datasets are rather small, all methods work in reasonable time.
For completeness, we show runtime results in milliseconds in Table \ref{running_time_result}.
As one may expect, computing \FT{Gradient-scores} is significantly faster than
computing SILO and LIME scores, which, in turn, is significantly faster
than computing SHAP scores.
\begin{table}[t] 
\caption{Runtime for different explanation methods (in ms).}
\label{running_time_result}
\vskip 0.15in
\begin{center}
\begin{small}
\begin{sc}
\begin{tabular}{clrrrr}
\toprule
& DATA & GRAD & SILO & LIME & SHAP \\

\midrule
          & PIMA  & \textbf{0.01} & 14.54  & 8.68  & 594.19  \\
    LR    & BKNT  & \textbf{0.01} & 13.30  & 30.99  & 182.63  \\
          & WDBC  & \textbf{0.02} & 14.10  & 18.74  & 3033.67  \\
        
\midrule
          & PIMA  & \textbf{0.24} & 13.10  & 9.60  & 603.36  \\
    MLP\FT{s}   & BKNT  & \textbf{0.32} & 13.31  & 33.33  & 178.32  \\
          & WDBC  & \textbf{0.28} & 14.28  & 18.46  & 3167.71  \\
         
\bottomrule
\end{tabular}
\end{sc}
\end{small}
\end{center}
\vskip -0.1in
\end{table}

\section{Conclusions}

We revisited faithfulness of attribution methods in the setting where
the prediction model is differentiable and features are continuous. 
While sampling and an empirical notion of faithfulness (based on the
empirical accuracy of a local substitute model) seem difficult to avoid
in the discrete setting, the continuous setting allows for
more analytical tools. We therefore proposed two analytical notions
of faithfulness (qualitative and quantitative) and showed that they 
are satisfied by \FT{Gradient-scoring} and, to some extent, only by 
\FT{Gradient-scoring}. In particular, quantitative faithfulness completely characterizes
Gradient-scoring in the sense that every explanation method that satisfies
this property must be equivalent to Gradient-scoring (Proposition \ref{thm_quant_faithfulness_characterization}). Roughly speaking, every attribution
method that guarantees that the scores are proportional to the changes
in the output of the classifier is either equivalent to gradient-scoring
or is based on an inferior approximation of the classifier.

To back up the theory, we investigated empirically to which extent
different attribution methods explain the true behaviour of logistic 
regression. Linear attribution methods like LIME and SILO do reasonably
well, but perform worse than \FT{Gradient-scoring}. SHAP does not
seem to capture the true behaviour of the classifier at all. For non-linear 
classifiers, there is still some correlation between linear attribution 
methods and the gradient. While it is hard to tell objectively which explanations are
better in this case, \FT{Gradient-scores} seem preferable. Not only
do they have a clearly defined analytical meaning and are not 
susceptible to noise caused by sampling, but they can also be 
computed 
easily. 

Overall,
we 
\FT{believe} that \FT{Gradient-scoring} is the most
suitable method for the continuous setting that we considered here. 
In particular, SHAP does not seem well suited for this setting as the SHAP scores for logistic regression seem completely uncorrelated with the actual weights.
In general, it seems that more research is necessary on clarifying
what feature scores under different methods actually represent. 
In particular, they can potentially be made more accurate by
adding scores to selected pairs (collections) of features
similar to how second-order (higher-order) terms approve the
accuracy of Taylor approximations.
However, naturally, there is a trade-off between accuracy of the
explanation, computational cost and comprehensibility.

\section*{Acknowledgements}
This project has received funding from the European Research Council (ERC) under the European Union’s Horizon 2020 research and innovation programme (grant agreement No. 101020934).


\bibliography{references_ICML2022}

\begin{thebibliography}{10}

\bibitem{adadi2018peeking}
{\sc Adadi, A., and Berrada, M.}
\newblock Peeking inside the black-box: a survey on explainable artificial
  intelligence (xai).
\newblock {\em IEEE access 6\/} (2018), 52138--52160.

\bibitem{Agarwal21}
{\sc Agarwal, S., Jabbari, S., Agarwal, C., Upadhyay, S., Wu, S., and
  Lakkaraju, H.}
\newblock Towards the unification and robustness of perturbation and gradient
  based explanations.
\newblock In {\em International Conference on Machine Learning {ICML}\/}
  (2021), M.~Meila and T.~Zhang, Eds., vol.~139 of {\em Proceedings of Machine
  Learning Research}, {PMLR}, pp.~110--119.

\bibitem{alcala2011keel}
{\sc Alcal{\'a}-Fdez, J., Fern{\'a}ndez, A., Luengo, J., Derrac, J.,
  Garc{\'\i}a, S., S{\'a}nchez, L., and Herrera, F.}
\newblock Keel data-mining software tool: data set repository, integration of
  algorithms and experimental analysis framework.
\newblock {\em Journal of Multiple-Valued Logic \& Soft Computing 17\/} (2011).

\bibitem{ancona2018towards}
{\sc Ancona, M., Ceolini, E., {\"O}ztireli, C., and Gross, M.}
\newblock Towards better understanding of gradient-based attribution methods
  for deep neural networks.
\newblock In {\em International Conference on Learning Representations\/}
  (2018).

\bibitem{baehrens2010explain}
{\sc Baehrens, D., Schroeter, T., Harmeling, S., Kawanabe, M., Hansen, K., and
  M{\"u}ller, K.-R.}
\newblock How to explain individual classification decisions.
\newblock {\em The Journal of Machine Learning Research 11\/} (2010),
  1803--1831.

\bibitem{bloniarz2016supervised}
{\sc Bloniarz, A., Talwalkar, A., Yu, B., and Wu, C.}
\newblock Supervised neighborhoods for distributed nonparametric regression.
\newblock In {\em Artificial Intelligence and Statistics\/} (2016), PMLR,
  pp.~1450--1459.

\bibitem{garreau2020explaining}
{\sc Garreau, D., and Luxburg, U.}
\newblock Explaining the explainer: A first theoretical analysis of lime.
\newblock In {\em International Conference on Artificial Intelligence and
  Statistics\/} (2020), PMLR, pp.~1287--1296.

\bibitem{ghorbani2019interpretation}
{\sc Ghorbani, A., Abid, A., and Zou, J.}
\newblock Interpretation of neural networks is fragile.
\newblock In {\em Proceedings of the AAAI Conference on Artificial
  Intelligence\/} (2019), vol.~33, pp.~3681--3688.

\bibitem{guidotti2021evaluating}
{\sc Guidotti, R.}
\newblock Evaluating local explanation methods on ground truth.
\newblock {\em Artificial Intelligence 291\/} (2021), 103428.

\bibitem{heo2019fooling}
{\sc Heo, J., Joo, S., and Moon, T.}
\newblock Fooling neural network interpretations via adversarial model
  manipulation.
\newblock {\em Advances in Neural Information Processing Systems 32\/} (2019),
  2925--2936.

\bibitem{lipton2018mythos}
{\sc Lipton, Z.~C.}
\newblock The mythos of model interpretability: In machine learning, the
  concept of interpretability is both important and slippery.
\newblock {\em Queue 16}, 3 (2018), 31--57.

\bibitem{lundberg2017unified}
{\sc Lundberg, S.~M., and Lee, S.-I.}
\newblock A unified approach to interpreting model predictions.
\newblock In {\em Proceedings of the 31st international conference on neural
  information processing systems\/} (2017), pp.~4768--4777.

\bibitem{lundberg2018explainable}
{\sc Lundberg, S.~M., Nair, B., Vavilala, M.~S., Horibe, M., Eisses, M.~J.,
  Adams, T., Liston, D.~E., Low, D. K.-W., Newman, S.-F., Kim, J., et~al.}
\newblock Explainable machine-learning predictions for the prevention of
  hypoxaemia during surgery.
\newblock {\em Nature Biomedical Engineering 2}, 10 (2018), 749.

\bibitem{nocedal2006numerical}
{\sc Nocedal, J., and Wright, S.}
\newblock {\em Numerical optimization}.
\newblock Springer Science \& Business Media, 2006.

\bibitem{plumb2018model}
{\sc Plumb, G., Molitor, D., and Talwalkar, A.}
\newblock Model agnostic supervised local explanations.
\newblock In {\em Proceedings of the 32nd International Conference on Neural
  Information Processing Systems\/} (2018), pp.~2520--2529.

\bibitem{ribeiro2016should}
{\sc Ribeiro, M.~T., Singh, S., and Guestrin, C.}
\newblock " why should i trust you?" explaining the predictions of any
  classifier.
\newblock In {\em Proceedings of the 22nd ACM SIGKDD international conference
  on knowledge discovery and data mining\/} (2016), pp.~1135--1144.

\bibitem{ribeiro2018anchors}
{\sc Ribeiro, M.~T., Singh, S., and Guestrin, C.}
\newblock Anchors: High-precision model-agnostic explanations.
\newblock In {\em Proceedings of the AAAI conference on artificial
  intelligence\/} (2018), vol.~32.

\bibitem{sanchez2015towards}
{\sc Sanchez, I., Rocktaschel, T., Riedel, S., and Singh, S.}
\newblock Towards extracting faithful and descriptive representations of latent
  variable models.
\newblock {\em AAAI Spring Syposium on Knowledge Representation and Reasoning
  (KRR): Integrating Symbolic and Neural Approaches 1\/} (2015), 4--1.

\bibitem{street1993nuclear}
{\sc Street, W.~N., Wolberg, W.~H., and Mangasarian, O.~L.}
\newblock Nuclear feature extraction for breast tumor diagnosis.
\newblock In {\em Biomedical image processing and biomedical visualization\/}
  (1993), vol.~1905, International Society for Optics and Photonics,
  pp.~861--870.

\bibitem{SundararajanN20}
{\sc Sundararajan, M., and Najmi, A.}
\newblock The many shapley values for model explanation.
\newblock In {\em International Conference on Machine Learning, {ICML}\/}
  (2020), vol.~119 of {\em Proceedings of Machine Learning Research}, {PMLR},
  pp.~9269--9278.

\bibitem{sundararajan2017axiomatic}
{\sc Sundararajan, M., Taly, A., and Yan, Q.}
\newblock Axiomatic attribution for deep networks.
\newblock In {\em International Conference on Machine Learning\/} (2017), PMLR,
  pp.~3319--3328.

\bibitem{szegedy2014intriguing}
{\sc Szegedy, C., Zaremba, W., Sutskever, I., Bruna, J., Erhan, D., Goodfellow,
  I., and Fergus, R.}
\newblock Intriguing properties of neural networks.
\newblock In {\em International Conference on Learning Representations
  (ICLR)\/} (2014).

\bibitem{van2021tractability}
{\sc Van~den Broeck, G., Lykov, A., Schleich, M., and Suciu, D.}
\newblock On the tractability of shap explanations.
\newblock In {\em AAAI Conference on Artificial Intelligence\/} (2021),
  pp.~6505--6513.

\bibitem{wachter2017counterfactual}
{\sc Wachter, S., Mittelstadt, B., and Russell, C.}
\newblock Counterfactual explanations without opening the black box: Automated
  decisions and the gdpr.
\newblock {\em Harv. JL \& Tech. 31\/} (2017), 841.

\bibitem{wang2020gradient}
{\sc Wang, J., Tuyls, J., Wallace, E., and Singh, S.}
\newblock Gradient-based analysis of nlp models is manipulable.
\newblock In {\em Conference on Empirical Methods in Natural Language
  Processing (EMNLP)\/} (2020), pp.~247--258.

\end{thebibliography}
\bibliographystyle{acm}


\end{document}